\documentclass{article}
\usepackage{ijcai23}

\usepackage{times}
\usepackage{soul}
\usepackage{url}
\usepackage{bbm}
\usepackage{amssymb}
\usepackage[hidelinks]{hyperref}
\usepackage[utf8]{inputenc}
\usepackage[small]{caption}
\usepackage{graphicx}
\usepackage{amsmath}
\usepackage{amsthm}
\usepackage{amsfonts}
\usepackage{booktabs}
\usepackage{algorithm}
\usepackage{algorithmic}
\usepackage{multirow}

\usepackage[switch]{lineno}
\usepackage{xcolor}

\newtheorem{defn}{Definition} % definition numbers are dependent on theorem numbers
\newtheorem{prop}{Proposition}

\newtheorem{hyp}{Hypothesis} 
\title{Goal Alignment: A Human-Aware Account of Value Alignment Problem}

\author{
Malek Mechergui
\and
Sarath Sreedharan
\affiliations
Colorado State University
\emails
\{Malek.Mechergui, Sarath.Sreedharan\}@colostate.edu
}

\begin{document}
% for citations
\newcommand{\Shortcite}[1] {\citeauthor{#1}~\shortcite{#1}}
\newcommand{\ShortciteManySameLead}[2] {\citeauthor{#1}~\shortcite{#1,#2}}

% for comments/discussion
\newcommand{\Xsays}[2] {\textbf{#1 says:} \emph{#2}}
\renewcommand{\Xsays}[2] {}
\newcommand{\Ssays}[1] {\Xsays{Sarath}{#1}}
\newcommand{\Psays}[1] {\Xsays{Pascal}{#1}}

% for the reductions
\newcommand{\qsatt}{\textit{QSAT}_2}
\mathchardef\mhyphen="2D
\newcommand{\thdnf}{3\mhyphen\textrm{DNF}}
\newcommand{\thcnf}{3\mhyphen\textrm{CNF}}

% action formalization
\newcommand{\Action}[2]  {\ensuremath{\mathit{#1}}}
\newcommand{\preP}[1]    {\ensuremath{\Action{pre}{#1}_+}}
\newcommand{\preN}[1]    {\ensuremath{\Action{pre}{#1}_-}}
\newcommand{\add}[1]     {\ensuremath{\Action{add}{#1}}}
\newcommand{\del}[1]     {\ensuremath{\Action{del}{#1}}}
\newcommand{\T}          {\ensuremath{\mathcal{T}}}
\newcommand{\sharedPreP} {\preP{}}
\newcommand{\sharedPreN} {\preN{}}
\newcommand{\sharedAdd}  {\add{}}
\newcommand{\sharedDel}  {\del{}}

% models
\newcommand{\RobotM}              {\ensuremath{\mathcal{M}^R}}
\newcommand{\HumanM}              {\ensuremath{\mathcal{M}^H}}
\newcommand{\ModelParaNegPrec}[2] {#1\textit{-has-neg-prec-}#2}
\newcommand{\ModelParaPosPrec}[2] {#1\textit{-has-pos-prec-}#2}
\newcommand{\ModelParaAdd}[2]     {#1\textit{-has-add-}#2}
\newcommand{\ModelParaDel}[2]     {#1\textit{-has-del-}#2}

\newcommand{\Oracle}
{\ensuremath{\mathcal{O}^{G^*}}}
\newcommand{\HAGL}
{\ensuremath{\mathcal{H}}}
\newcommand{\VQ}
{\ensuremath{\mathcal{V}^Q}}
%highlighting new parts
\newcommand{\note}{\textcolor{red}}
\maketitle
\begin{abstract}
Value alignment problems arise in scenarios where the specified objectives of an AI agent don’t match the true underlying objective of its users. The problem has been widely argued to be one of the central safety problems in AI. Unfortunately, most existing works in value alignment tend to focus on issues that are primarily related to the fact that reward functions are an unintuitive mechanism to specify objectives. However, the complexity of the objective specification mechanism is just one of many reasons why the user may have misspecified their objective. A foundational cause for misalignment that is being overlooked by these works is the inherent asymmetry in human expectations about the agent's behavior and the behavior generated by the agent for the specified objective. To address this lacuna, we propose a novel formulation for the value alignment problem, named {\em goal alignment} that focuses on a few central challenges related to value alignment. In doing so, we bridge the currently disparate research areas of value alignment and human-aware planning. Additionally, we propose a first-of-its-kind interactive algorithm that is capable of using information generated under incorrect beliefs about the agent, to determine the true underlying goal of the user.
\end{abstract}

\section{Introduction}

% Talk about the fact that this is actually connected to everyday interactions

\begin{figure}[ht]
\includegraphics[scale=0.5]{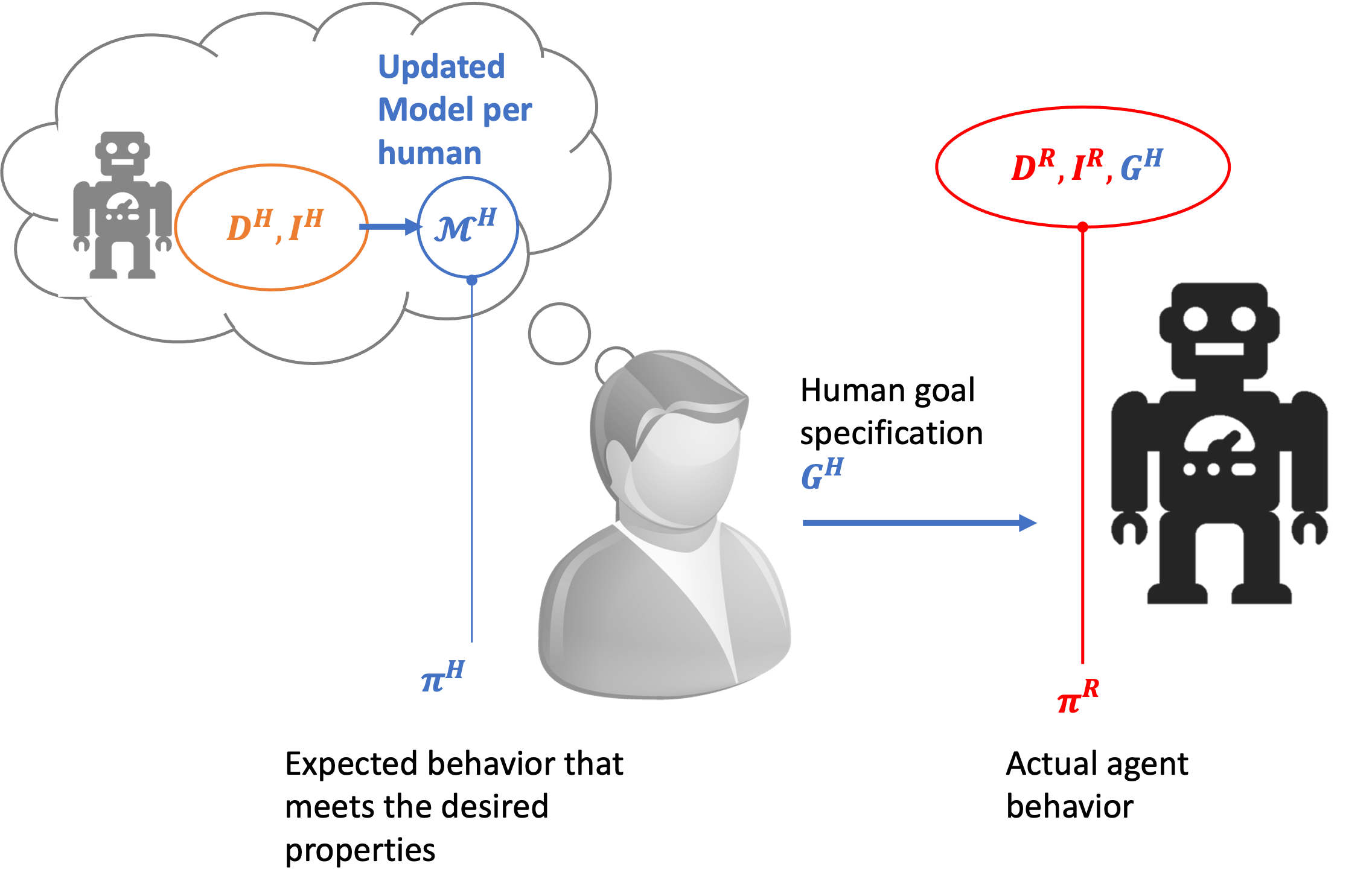} 
\caption{A diagrammatic overview of the objective specification process as contextualized in a
generalized Human-aware AI framework. Humans ascribe a domain model and initial state to the agent which may be different from the true model. Now the human identifies a goal specification whose inclusion in the agent's model they believe will result in plans they would prefer. Note that the human is generating the model updates based on a potentially incorrect understanding of the system's model and using possibly faulty reasoning. The resulting outcomes from pursuing that goal using the robot model could be very different from what the human expected.} 
\label{fig1}
\end{figure}

Value alignment, as presented in \cite{cirl}, is the problem of ensuring that an AI agent's pursuit of its specified objectives will maximize or satisfy the true underlying objective of its human user. Usually studied in the context of scenarios, where such misalignments could have catastrophic consequences, the problem has been widely argued to be one of the most important problems related to AI safety \cite{book1,book2}. While there is a general consensus that the primary cause of the value misalignment problem is the user’s failure to correctly anticipate the outcomes of their specification, current works tend to focus on addressing only some aspects of the problem.
% current methods focus
In particular, most works within value alignment tend to focus on decision-theoretic settings, where the objectives are specified as reward functions and try to address problems closely connected to the nature of this representation scheme (cf. \cite{cirl,rewardModel,IRD}).

We argue that, the extant literature on value alignment overlooks the fundamental problem that any information user provides to the system is going to be skewed by 
%their inherent limitations in inferential capabilities and 
their beliefs about the agent model, which may be different from the agent's own model.
Which in turn means that the user's expectation about the behavior the agent would exhibit in response to a particular goal specification could be drastically different from what might actually be followed.
Arguably, this asymmetry between the user's expectations about agent behavior and the agent's true behavior is one of the main factors that gives rise to the misalignment in the first place.
As such, for a system to correctly use any information provided by the user it must try to re-interpret it in the light of this inherent difference between the user and the agent. 

Thus in this paper, we will present a new formalization of the value alignment problem that accounts for this asymmetry between the user and the AI agent. We will do so by first removing many of the extraneous parts of the problem that are artifacts of the setting rather than the true nature of the value misalignment problem. In fact, we will focus on one of the most basic sequential decision-making setting, namely deterministic goal-directed planning. This setting will transform the value alignment problem to a {\em goal alignment problem}, which will be specifically grounded in a scenario where the user's belief could be different from the agent model.

To achieve this, we will build on and generalize a framework called Human-Aware AI \cite{book}, that was originally introduced to generate explainable behavior. The framework uses psychological concepts of mental models \cite{tom}, to model and understand human-AI interaction. Figure \ref{fig1}, shows how we could build on the human-aware AI framework to understand how goal misspecification may arise. As clearly illustrated, the human is specifying a goal to an agent to elicit a behavior they would deem desirable. However, if their beliefs about the agent model are different from the true agent model or if their reasoning process is faulty, it could lead to the human providing goals that may result in completely unexpected behaviors. This also means that if the agent hopes to identify and try to satisfy the true objectives of the user, it must identify the existing differences between the user's beliefs and the agent model and use this difference to reason about the intended behavior.

In summary, the primary contributions of this paper are as follows:
\begin{itemize}
    \item We formalize and define the problem of {\em Human-aware goal alignment}; a formulation of the value alignment problem that explicitly accounts for the asymmetry between the user's expectations and the agent's decisions. 
    \item We establish the lower bound complexity of the human-aware goal alignment problem.
    \item We introduce, a first-of-its-kind interactive goal elicitation algorithm, that can use information generated from incorrect model beliefs.
    \item We provide an empirical evaluation demonstrating the computational characteristics of our algorithm.
\end{itemize}

\section{Background}
\label{sec:backg}
We will be focusing on deterministic goal-directed planning problems.
Such problems can be represented using a tuple of the form 
$\mathcal{M} = \langle D, I, G\rangle$ \cite{Geffner2013PlanningIntro}. Under this notation, $D$ corresponds to the domain model of the planning problem, which is further defined by using a tuple, $D = \langle F, A\rangle$, where $F$ is a set of propositional fluents that are used to define the state space of the planning problem and $A$ provides the set of actions that can be executed by the agent. Each state possible under the given planning problem can be uniquely identified by the set of fluents that are true in that state, thus the total number of possible states is equal to $2^{|F|}$. Finally, $I$ corresponds to the start state and $G$ captures the partial goal specification, such that any state $s \supseteq G$ is considered a valid goal state.

Now each action $a \in A$ is further defined by the tuple, $a = \langle \preP{\mathcal{M}}(a), \add{\mathcal{M}}(a), \del{\mathcal{M}}(a)\rangle$, where $\preP{\mathcal{M}}$ are the preconditions that need to be satisfied to execute $a$, while $\add{\mathcal{M}}$ and $\del{\mathcal{M}}$ denote the add and delete effects related to the action. We will use $\T$ to capture the effects of executing an action at a given state
% Models as satisfaction
$\T(a, s, D)$ defined as:
\begin{equation*}
= \begin{cases} (s\setminus\del{\mathcal{M}}(a))\cup \add{{\mathcal{M}}}(a), & \textrm{if } \preP{\mathcal{M}}(a) \subseteq s \\
              \emph{undefined} & \textrm{otherwise}
\end{cases}
\end{equation*}
Overloading the notations a little bit, we will also use $\T$ to capture the consequence of executing a sequence of actions $<a_1,a_2,..,a_k>$, i.e., 
\begin{align*}
    \begin{split}
        \T(<a_1,a_2,..,a_k>, s, D) = \\ \T(a_1, \T(<a_1,a_2,..,a_k>, s, D), D).
    \end{split}
\end{align*}

A solution to a planning problem takes the form of a plan, where a plan is a sequence of actions whose execution in the initial state would result in a goal state, i.e., $\pi=\langle a_1,..., a_k\rangle$ is a plan if $\T(\pi, I^{\mathcal{M}}, D^{\mathcal{M}})\supseteq G^{\mathcal{M}}$. We can additionally, associate a cost with each action, however, to keep the formulation simple we will simply assume that each action has a unit cost and $C(\pi) = |\pi|$. We will refer to a plan $\pi$ as being optimal if there exist no other valid plans that cost less than $C(\pi)$.

\section{Related Work}
\label{sec:relw}
The recognition of potential dangers of misspecification of agent objectives has a long history within AI \cite{turing,weiner}, and builds on ideas from even earlier philosophers. However, the modern form of the problem was effectively established by \cite{cirl}, where they formalize the notion of assistive games to help optimize for the human’s unspecified objective. Apart from the formalization, one of the core technical contributions of the paper was the development of an algorithm to help generate more informative traces. However, as we will see such information would be influenced by not only their inability to perform correct introspection (commonly acknowledged in the literature), but also their misunderstandings about the agent itself. Other prominent works in this direction include works on reward design \cite{IRD}, works that  try to query the human about preferred behavior \cite{rewardModel} and other works on generating informative traces \cite{pedagogic}. There are also works that investigated the moral aspects of value alignment \cite{geometric,leike_2022}, however, we will treat the problem of developing moral agents as being orthogonal to the problem of aligning objectives.

None of these works explicitly try to model the role played by the human and agent asymmetries in causing this misalignment in the first place. Human-aware AI  \cite{book}, was a framework that was originally developed in the context of generating explainable behavior. 
The framework hypothesizes that potential asymmetries between the human and the AI agent can cause a mismatch between the decisions chosen by the system and what the human would have expected. Such mismatches would cause the human to be confused as to why the agent may be following a particular action, which in turn would require the agent to explain its current decisions to the user. In general, these works identify three broad classes of asymmetries between the user and the agent \cite{thesis}, namely asymmetry in knowledge about the task, asymmetry in inferential capabilities, and asymmetry in vocabulary. The explanation methods developed under the aegis of human-aware AI (cf. \cite{model-rec,helm,blackbox}) tend to focus on identifying and addressing these asymmetries so that the agent and the user can reconcile their differences in expectations about the right course of action for a given problem. In many ways, the goal of this work is to invert the process. We are trying to identify and leverage asymmetries to reconstruct and then try to meet the original expectations the human had, from the information they provide. In this sense, our work is also closely related to a method called explicable planning \cite{expyu}, where the system tries to generate behavior that matches user expectations. 
However, in explicable planning, the final goal is usually provided and the objective of the planning process is to generate plans that closely match behaviors that the human expected. In our case, we will not try to match the generated behavior with what the human expects, but rather focus only on ensuring that the outcomes we generate satisfy what the user expected (the behavior that generates that outcome may look nothing like what the user expected). 

A parallel thread of work in value alignment that is orthogonal to the efforts outlined in this paper is that of formulating the set of values that the agent needs to be imbued with (cf. \cite{Lera-LeriBSLR22,Serramia1,MontesS22} ). These works build on notions of values as determined in the wider psychological and social sciences literature \cite{ref9,ref10}. Our method is completely compatible with these efforts, as our objective is to ensure how these values, once identified, can be enforced in the agent. Our framework as of right now makes no commitments as to what goals or objectives are specified by the user.

Another closely related set of works is that of model elicitation \cite{ref8old,aineto2019model}, preference elicitation \cite{mantik2022preference,chen2004survey}, resolving reward uncertainty \cite{zhang2017approximately,wilson2012bayesian}, goal refinement \cite{mohajeriparizi2022preference} and the technique of knowledge tracing \cite{corbett1994knowledge} as applied in the context of intelligent tutoring systems. All these works are trying to solve a closely related problem, in that they are trying to acquire some model information from a user or another agent. However, such works are fundamentally incompatible with our setting as none of the works in these areas currently allow the system to leverage information generated by users under potentially incorrect beliefs about the system.

\section{Motivating Example}
\label{sec:motiv}
Consider an intelligent robotic assistant that is being used to help in daily household chores of its users. 
The robot is expected to take task specification, along with any optional guidance from its users and is expected to fulfill the user’s requirements.
Let us assume that in this case, the robot is aware that the goals that the user may specify may be incomplete.
As a specific example, consider a case where the user asks the robot to prepare a cup of tea.
If the robot were to simply opt for the optimal plan, it would have simply reached out to the tea leaves closest to it and made tea with it. Which in this case turns out to be some low-quality tea leaves left at the bottom of the kitchen cupboard.
However, if the robot was to follow this plan, the prepared tea wouldn’t have satisfied the user’s expectations since when asking for a cup of tea the user was actually hoping to get tea made with good quality tea-leaves. The user may have just forgotten to specify the quality or overlooked the possibility that the tea could have been made with poor quality tea-leaves.

Now the robot on its own can’t come up with what the human may have really wanted, and querying them about all other possibilities might be extremely difficult. Thankfully, in this case the human may have or is willing to provide additional instructions about the task. Let’s assume the simplest case where the human provides an entire plan on how to make the tea. 
Let’s assume that the plan provided involves the robot fetching a ladder, putting it next to the cupboard, climbing on the ladder and fetching good quality tea leaves, then making the tea. 
This is not a plan the robot can execute on its own, since unbeknownst to the user, the robot can’t climb ladders. However, assuming this plan, at least in the human model, captures what they really want could give the robot clues about the true human goal. Once this is determined, the robot can independently figure out how to achieve the goal.

Specifically, if it knew the human’s belief about the robot, it could try to simulate the plan in the human model and see what state they expect and try to see what fluents that are true in the goal state may additionally be part of the true human goal. Now in this case, this could involve the fluent regarding the use of high quality tea leaves, but also fluents about the position of the ladder and whether the robot used it. Now one of the central challenges involved with this setting is to come up with a method wherein the robot finds a plan that is guaranteed to satisfy the unspecified human goal while minimizing the number of times the human is queried to get more information.

\section{Goal Alignment Problem}
\label{sec:prob}
Our setting consists of a robot (we use the term robot as a stand-in for any autonomous agent) that is expected to perform a task assigned to it by a human. Now we will start by denoting the domain model used by the robot as  $D^{R}= \langle F, A^{R}\rangle$, and the initial state as captured by the robot as $I^{R}$. Now, keeping with the conventions from human-aware AI, the human who assigns the task may have different beliefs about the robot's model and the current state. Such differences could reflect their potential biases about the robot and their own incorrect and limited understanding of the task. Let us denote the human's beliefs about the robot model as  $\HumanM = \langle D^{H}, I^H, G^H\rangle$, where $D^{H}= \langle F, A^{H}\rangle$ is domain model human ascribes to the robot,  $I^H$ the human belief about the initial state and $G^{H}$ is the goal specified by the human.
The human would have come up with this goal specification while keeping in mind their belief about the robot's capability and the human's own preferences about the expected outcome.
In our earlier example, $G^{H}$ would just include the fact that tea has to be made. 
The assumption that both the human and the robot share fluents is a common assumption made throughout human-aware planning problems (cf. \cite{book}), and we can leverage methods like \cite{blackbox} to easily relax this assumption.
The value alignment problem arises when optimization of the specific robot objective doesn't necessarily maximize the underlying human reward. In our setting, this translates to the possibility that a plan that achieves the specified goal need not achieve the underlying human goal. 
Going back to our example, the goal specification that a tea needs to be made is misaligned because there are plans that are valid to that goal and which do not satisfy other considerations the human could have, like the fact that the tea needs to be made with high-quality tea leaves.
More formally, we will define the goal-misalignment problem as follows

\begin{defn}
A goal specification $G^{H}$ is said to be misaligned with the human goal $G^*$ for a robot domain model $D^{R}$ and initial state $I^{R}$, if there exists an action sequence $\pi = \langle a_1,...,a_k\rangle$  such that $\T(\pi,I^{R},D^{R}) \supseteq G^{H}$, but $\T(\pi,I^{R},D^{R}) \not\supseteq G^{*}$
\end{defn}

Traditionally one of the main sources of information used to address value alignment problems (cf. the setting presented by \shortcite{cirl}), are potential traces provided by humans that satisfy their underlying objectives.
The use of such information generally entails the assumption that, while the human may not be able to correctly specify their objectives, they can still recognize when a state that satisfies their objectives is reached and potentially reason about how to reach such states.
In our case, this information is contained within the human-specified plan $\pi^{H}$, that the human believes the robot can follow to achieve the goal\footnote{Equivalently, we could also consider cases where the human may provide a plan they could execute themselves to achieve the goal. In such case, the remaining problem definition and solution approach remain the same except that we will be using the human model of themselves ($D^{H}$) instead of their model of the robot ($D^{R}$) to analyze the plan.}. In our example, this would correspond to the plan provided by the user involving the use of ladders.

In theory, the simplicity of the setting dissipates almost all of the traditional challenges that are identified by current solutions to the value alignment problem. For one, goals are a much simpler structure to specify objectives than rewards are. The complexity of rewards as a specification mechanism is the primary focus of many approaches like \cite{IRD} and \cite{rewardModel} and there is empirical evidence showing people are bad at specifying effective reward functions. On the other hand, there is psychological evidence that argues that people tend to perform planning in terms of goals and subgoals \cite{simon}. As such, people would have a much easier time specifying goals than rewards. Similarly, for a deterministic task, a single plan is sufficient to reach the goal. Unlike \cite{cirl}, we need not worry about using inverse-reinforcement learning algorithms to identify the more general reward function that may be implied by the trace.

However, the clarity of the setting also affords us the opportunity to see the more foundational problems that are frequently shrouded by the complexity of the setting. First off, even in this rather simple setting, the human's ability to effectively specify objectives depend on their correct understanding of the robot’s capabilities and their ability to correctly anticipate the kind of plans that the robot may come up with in response to this new goal. This could even include cases where the limitations of the inferential capabilities of the human prevent them from correctly anticipating the effects. This inability to correctly model the robot lies at the heart of the value alignment problem, in fact, \cite{kimtalk} presents a human interaction with a modern AI agent to that of interacting with an alien intelligence. 

Now coming back to the plan  $\pi^{H}$, even if we allow for the possibility that in the human mental model that the plan could achieve the true goal, there is no reason to believe that the robot can execute it or even that executing it will result in the same goal state. In our running example, the robot can't execute the specified plan as it will not be able to execute the climb ladder action.
As a starting point, we will assume that what the human really cares about is the final outcome of a plan, and thus effectively only the goal state matters. 
Thus a new possibility may be to try to not follow the specified plan, but rather try to recreate the final state expected by the human.
Here again, we run into a new problem, as the robot may not be exactly able to generate the state that results from executing the plan in the human mental model.
In our running example, let's assume there are fluents corresponding to what tools the robot used. In this case, it will not be able to exactly replicate the final state as it can't climb the ladder and thus can't turn the fluent related to the ladder being used true.
Note that this is completely consistent with cases where the human may have trajectory level constraints, as they can be compiled down into goal state fluents (cf. \cite{BaierBM09}). Now let the unknown goal the human has, be $G^*$ and they only partially specified it to the robot, i.e., $G^{H} \subseteq G^*$. Thus, the central challenge is how does the system determine if it can achieve $G^*$, and if so how does it come up with a plan that satisfies the goal $G^*$. 

However, the fact that the human provided the robot with a plan gives us information about what $G^*$. For one, we can assert that $G^*$ must be a subset of what the human believes would have resulted from executing the plan ($\T(\pi^{H}, I^{H}, D^{H})$). The problem of course is how does one identify the exact subset. The fact that goals are an intuitive structure for humans means that we can directly query the human about them. Unfortunately, queries designed to directly get $G^*$ (say by asking, {\em `are you sure you only need me to achieve $G^{H}$?'}) are bound to fail. This is because the difference between $G^{H}$ and $G^*$, is not just a result of them forgetting some fluents, but a reflection of their beliefs about the task.
For example, in the tea-making task, the human would never remember to specify that the tea needs to be made with water because they would never be able to imagine doing it in any other way.
However, the robot could on the other hand ask the human whether they care about any given fluent (for example, {\em 'would you mind if the tea was not made with water?'}). Thus we will introduce a function $\Oracle: F \rightarrow [0,1]$ that will return 1 if a given fluent is part of $G^*$. 
Note that the central computational challenge we have is to find plans that will achieve the goal while minimizing the queries to humans.
Now with all the components specified, we are ready to formally define the central problem.

\begin{defn}
A \textbf{human-aware goal alignment (HAGL)} is specified by the tuple $\HAGL = \langle D^{R}, I^{R}, G^{H}, D^{H}, I^{H}, \pi^{H}, \Oracle \rangle$, where there exists an unknown goal $G^*$, such that $\T(\pi^{H},I^{H},D^{H}) \supseteq G^*$ and $G^{H} \subseteq G^*$ and  $\forall, f\in F, \Oracle(f) = 1$, if and only if $f \in G^*$. Now the goal of the robot is to find $\pi^{R}$ such that $\T(\pi^{R},I^{R},D^R) \supseteq G^*$, if one exists, while minimizing the queries to $\Oracle$
\end{defn}

As with many of the human-aware planning works, we will assume access to  $D^{H}$ and $I^{H}$.
Note that the solution we propose of finding a plan that results in a superset of  $G^*$ is still consistent with cases where the human may want to avoid some undesirable side effects. This can be achieved by adding new fluents that correspond to negations of existing fluents (similarly the model could be updated to ensure that the original fluent and the new fluent will always carry complementary values in every reachable state). 
Our current formulation can capture cases where a fluent corresponds to an undesirable side-effect, by adding the fluent corresponding to the negation of the undesirable fluent into the goal specification $G^*$.

Now just to see the complexity of the specified problem, we can compare it against planning and see that it is at the very least as hard as solving classical planning problems, i.e., it is at least PSPACE-Hard.

\begin{prop}
A decision-version of HAGL, i.e, the problem of establishing whether there exists a plan for a given a HAGL problem $\HAGL$ that satisfies $G^*$ with just $K$ queries to \Oracle, is at least PSPACE-Hard.
\end{prop}
\begin{proof}[Proof Sketch]
We can establish this by showing that a plan existence problem for a model $\mathcal{M} = \langle D, I, G\rangle$ (which is known to be PSPACE-Complete \cite{bylander}) can be compiled into a HAGL problem. Specifically, one where $G^*$ is the same as $G$, the robot domain model and initial state are the same as those that are part of the original planning problem and the human model contains an action $a^{G}$ with an empty precondition that sets the $G$ true. Here the human plan is given as  $\pi^{H} = \langle a^{G}\rangle$ and we can additionally set $K = |F|$. Now the original planning problem is solvable if and only if there exists a plan for the HAGL problem.
\end{proof}

This further highlights our argument that even when one removes many of the traditional complexities associated with value alignment, we still find a complex and challenging computational problem at the heart of the goal-alignment problem. One that could have clear implications on everyday interactions humans could have with AI systems.

One of the big advantages that this formulation has over the traditional ones is the fact that $\T(\pi^{H},I^{H},D^H)$ already gives you an upper bound on possible things the human goal may contain. In fact, if the robot can already achieve a state that is a superset of $\T(\pi^{H},I^{H},D^H)$, then that plan is guaranteed to be a plan that satisfies the true human goal. This is only possible because the robot is maintaining an explicit model of the human’s belief about the robot model. However, this is only one way in which modeling human beliefs can help the robot in finding plans that satisfy the true human goal. As we will see in the next section, we can further leverage the human model to get better estimates on which of these goal fluents the human may have actually intended to achieve (as opposed to mere unintended side-effects).

\section{A Solution for Goal Alignment Problem}
\label{sec:solutn}
In addition to introducing a new version of the value alignment problem, we will also propose a solution for the goal alignment problem as described earlier. In particular,  we will approximate the value of information related to querying each fluent and then iteratively query the ones with the highest value. 
We will only use this procedure if $G^{H}$ is achievable, but the robot can't achieve all the fluents that were made true by the human plan in the human model ($ \T(\pi^{H},I^{H},D^{H})$).
We will calculate the value associated with querying about each fluent, as 
\[\VQ(f) = p(f\in G^*) \times V(f\in G^*) + (1 - p(f\in G^*)) \times V(f\not\in G^*)\]

Where $p(f\in G^*)$ is the probability that fluent is part of the goal and $V(f\in G^*)$, respective values of knowing whether $f$ is part of the goal or not. Now to simplify the calculation of these components, we will make a simplifying assumption that the achievement of each fluent can be done independently of each other. 
Let $S^{H}_{G^*}$ represent the state that results from executing the plan $\pi^{H}$ in the human model (i.e., $S^{H}_{G^*} = \T(\pi^{H},I^{H},D^{H})$) and 
let $\hat{F} \subseteq S^{H}_{G^*}$ be the set of fluents in the goal state that the robot cannot achieve in its true model.
Now to calculate the probability, we will employ a strategy similar to the ones used in goal recognition \cite{RamirezG10}. Namely to detect whether the suboptimality of the plan specified by the human may be explained by a given fluent. That is if the inclusion of a fluent $f$ in the goal set (i.e., $G^{H} \cup \{f\}$), makes the optimal plan for the new goal in the human model closer to the cost of the specified plan, then you will assign a higher probability to that fluent. Keeping with the conventions used by   \cite{RamirezG10}, we can formalize this as

\[ p(f\in G^*) \propto e^{-1\times\beta\times|C(\pi^{H}) - C(\hat{\pi}^*_f)| }\]
Where $\hat{\pi}^*_f$ is a plan that is optimal in the human model for the goal $G^{H} \cup {f}$, where $\beta$ is usually referred to as a rationality parameter and controls the randomness of the decision-maker. 
Note that this approach assumes that the human follows a noisy rational decision-making process, an assumption that has been shown to have psychological validity \cite{noisyrat}.

Now coming to the value, the value function reflects the certainty the robot has regarding the achievability of the goal state. 
If the robot knows for certain that it can be achieved or cannot be achieved then it will be set to 1. 
More formally the value will be equal to the sum of the probability that the $G^*$ is unachievable and the probability there exists a single plan that achieves $G^*$ (these two terms are mutually exclusive). 
Now we can find a lower bound on this true value by just using the probability that the goal is unachievable. 
\[ V(f\in G^{*}) \approxeq \sum_{\bar{G}} P(G^{*} = \bar{G}) \times \mathbbm{1}(\bar{G}~\textrm{not solvable})\]
Where $\bar{G}$ is any subset of $S^{H}_{G^*}$ containing $G^{H}$ that satisfy $f\in G^{*}$ (i.e., $G^{H} \subseteq \bar{G} \subseteq S^{H}_{G^*}$ and $f \in \bar{G}$), $P(G^{*} = \bar{G})$ probability that the true goal is the same as $\bar{G}$ and $\mathbbm{1}(\bar{G}~\textrm{not solvable})$ is an indicator function that evaluates to true if $\bar{G}$ is unsolvable. We can similarly define $V(f \not\in G^{*})$, but now we will only consider subsets of goal state that don't contain $f$.

Exactly calculating this lower bound on true value can still be computationally expensive, as it would require effectively testing the achievability of every subset that satisfies the condition discussed above (and calculating the probability as well).
However, since we are assuming that if a fluent is achievable in isolation in robot model, it can also be achieved as part of any goal state, we only need to care about the fluents that are part of $\hat{F}$
So we will define
\[\tilde{V}(f\in G^*) = \begin{cases} 1 ~\textrm{if }~f~ \textrm{is not achievable}\\ \prod_{\hat{f} \in \hat{F}}   p(\hat{f}\in G^*)~\textrm{Otherwise}\end{cases}\]
In the case of $\tilde{V}(f\not\in G^*)$ the value is always given as $\tilde{V}(f\not\in G^*) = \prod_{\hat{f} \in \hat{F}\setminus \{f\}}   p(\hat{f}\in G^*)$. Now the important point of this approximation is the assumption that each fluent’s independent achievability reflects its overall achievability. However, while many fluents may be achievable in isolation, there may be subsets of fluents containing that fluent which are not achievable. 
However, we can show that the value we calculated is guaranteed to be an approximation of the true value. 

\begin{prop}
For a given HAGL problem for an $f \in S^{H}_{G^*}$, we will have $V(f\in G^*)\geq \tilde{V}(f\in G^*)$ and $V(f\not\in G^*)\geq \tilde{V}(f\not\in G^*)$
\end{prop}
\begin{proof}[Proof Sketch]
This follows from two facts (a) $P(G^* =\bar{G}) \geq P(\bar{f} \in G^*)$ for any $\bar{f} \bar{G}$, and (b) there may be subsets of $S^{H}_{G^*}$ that are unsolvable, which doesn't contain any elements from $\hat{F}$. This means the sum of elements used to calculate the lower bound $V(f\in G^{*})$ would be greater than or equal to $\tilde{V}(f\in G^{*})$. When $\tilde{V}(f\in G^{*})=1$, then $V(f\in G^{*})$ must equal to one, since all possible goals are unachievable and when $\tilde{V}(f\in G^{*})=\prod_{\hat{f} \in \hat{F}}   p(\hat{f}\in G^*)$, then there must exist at least one term in the sum that is greater than or equal to $\prod_{\hat{f} \in \hat{F}}   p(\hat{f}\in G^*)$. We can use a similar kind of reasoning to show the relation also exists between $V(f\not\in G^*)$ and $\tilde{V}(f\not\in G^*)$.
\end{proof}

Now that we have a value associated with each fluent. We will start by querying them in the order of their value. We will end the query process under one of the three conditions
\begin{enumerate}
    \item The human says yes to a fluent that cannot be achieved
    \item The current subset of fluents the human has said yes to cannot be achieved along with the goal
    \item There exists a plan that can achieve the current subset of fluents the human has said yes to can be achieved along with $G^{H}$ and any unqueried fluent.
\end{enumerate}
The first two conditions correspond to cases where the robot can't achieve the expected goal and the latter where the robot can achieve a superset of $G^*$ and thus that plan would be acceptable to the human. Algorithm \ref{algo:overall} presents the pseudocode for the overall procedure.

\begin{algorithm}[tb]
   \caption{An approximation-based algorithm to find a solution to a HAGL}
   \label{algo:overall}
   %\small
\begin{algorithmic}
   \STATE {\bfseries Input:} $\HAGL = \langle D^{R}, I^{R}, G^{H}, \pi^{H}, \Oracle \rangle$
   \STATE $S^{H}_{G^*} = \T(\pi^{H},I^{H},D^{H})$
   \IF{$\langle D^{R}, I^{R}, G^{H}\rangle$ not solvable}
   \RETURN No plan exists
   \ENDIF
   \IF{$\langle D^{R}, I^{R}, S^{H}_{G^*}\rangle$ is solvable}
   \RETURN Return a valid plan for $\langle D^{R}, I^{R}, S^{H}_{G^*}\rangle$
   \ENDIF
   \STATE $Q \leftarrow$ A queue of fluents from the set $S^{H}_{G^*}\setminus G^{H}$ ordered by $\VQ$
   \STATE $\mathbb{C} \leftarrow \emptyset$
   \WHILE{$Q$ is not empty}
       \STATE $f \leftarrow Q.pop()$
       \IF{\Oracle(f) == 1}
       \STATE $\mathbb{C} = \mathbb{C} \cup \{f\}$
       \IF{$\langle D^{R}, I^{R}, G^{H} \cup \mathbb{C}\rangle$ not solvable}
            \RETURN No plan exists
        \ENDIF
       \ELSE
          \STATE $\hat{G} = G^{H} \cup \mathbb{C} \cup Q$
          \IF{$\langle D^{R}, I^{R}, \hat{G}\rangle$ is solvable}
            \RETURN Return a valid plan for $\langle D^{R}, I^{R}, \hat{G}\rangle$
        \ENDIF
       \ENDIF
       \ENDWHILE
    \IF{$\langle D^{R}, I^{R}, G^{H} \cup \mathbb{C}\rangle$ not solvable}
        \RETURN No plan exists
    \ELSE
        \RETURN Return a valid plan for $\langle D^{R}, I^{R}, G^{H} \cup \mathbb{C}\rangle$
    \ENDIF
\end{algorithmic}
\end{algorithm}

\begin{prop}
    Algorithm \ref{algo:overall} is a complete procedure for any given HAGL problem, i.e, it will always find a solution if one exists.
\end{prop}

This result follows from the fact that in the worst case, it would ask about every fluent that is part of $S^{H}_{G^*}$ and will be able to determine if a plan exists or not.

In the case of the running example, the $\hat{F}$ only consists of the fluent corresponding to the use of the ladder. The fluents corresponding to the use of the ladder and the use of the high-quality tea leaves will be assigned the highest probability. In this case, the proposed algorithm generates a plan that achieves the remaining goal fluents once the human is queried about whether the ladder used is part of the goal. Averaged across ten runs, we found that for the running example, our algorithm will query 4.2 times (with the maximum number of queries being 8).  
\section{Empirical Evaluation}
\label{sec:eval}
For evaluating our proposed algorithm, we ran our method on a set of problems selected from standard IPC benchmark problems \cite{ipc}.
Our primary motivation was to test the effectiveness of our method in reducing the number of times the user would need to be queried before the true goal is found. Since we are unaware of any existing methods we can directly apply in this setting, we will compare the number of queries generated against a simple baseline that would query the user about all potential goal predicates. Specifically, the hypothesis we will test will be
\begin{hyp}
The average number of queries generated by our algorithm will be lower than the naive upper bound on the number of queries, which is equal to $|S^{H}_{G^*} \setminus G^{H}|$.
\label{hyp1}
\end{hyp}
In particular, we considered five domains, namely, Blocksworld, Driverlog, Elevators, Rover and Logistics. For each domain, we selected five instances that were used in previous competitions. The true goal in this case consisted of the goal that was specified as part of the original problem, while we created the goal specification provided to the robot by randomly deleting a predicate from the goal specification. The human model was formed by randomly deleting preconditions and deletes from the original domain description and we used the original domain description as the robot model. All plans were generated using FastDownward planner \cite{FD} and we used A-star search with LM-cut heuristic \cite{LMCut} and set $\beta$ to one for probability calculation. All experiments were run on a linux machine with 32GB ram and 16 Intel(R) Xeon(R) 2.60GHz CPUs. We ran our algorithm on each problem instance ten times and the results from our evaluation are provided in Table \ref{tab1}.
The second column in Table 1, provides the baseline upper bound on the number of queries and the second and third columns list the average number of queries generated and the average time taken by our algorithm (along with their standard deviations).

The most striking result to note is the fact that, apart from the blocksworld domain, we see a significant drop in the number of queries in almost all the domains. In fact, for many of the problems the algorithm doesn’t even need to generate a single query to identify a plan that is guaranteed to satisfy the user’s hidden goal. 
This means that for these problems our method was able to find a plan which could achieve a superset of the goal state expected by the user.
The cases where the gains are less marked, particularly in Blocksworld, seem to correspond to ones where the number of fluents in the goal states are small. This indicates that our method will be most effective in problems with a larger fluent set and by extension a larger state space. This is a particularly useful property, as a naive querying strategy will not be viable in such problem settings. It is also worth noting that the time taken to complete the whole interaction is short and within an acceptable bounds for real-time interaction with users.

\begin{table}[!t]
\small
\centering
  \begin{tabular}{|r|r|c|c|c|c|r|r}
    \toprule
    \multicolumn{2}{|c|}{Problem Instance} &
     $|S^{H}_{G^*} \setminus G^{H}|$ &   \multicolumn{2}{|c|}{No of Queries}
    & \multicolumn{2}{c|} {Time (secs)}\\
    \multicolumn{2}{|c|}{} &
     &  Mean & Std
    & Mean & std\\
          \midrule
    %   \midrule

    \multirow{5}{*}{Blocks} &p1& 7 & 6.4&  1.1 & 5.08 & 0.37\\ 
&p2&3 & 2.6 & 0.52& 2.72 & 0.2 \\
&p3&7 & 5.9 & 1.1 & 4.9 & 0.37\\
&p4& 4& 3.8& 0&3.37 & 0.1\\
&p5&8 & 7.3 & 1.1& 5.6 & 0.24\\
    \bottomrule
 \multirow{5}{*}{Driverlog} 
&p1& 21 &0 &0&0.81 & 0.03\\
&p2 &24 &0 & 0&1 & 0.02\\
&p3&26 &0& 0& 0.83 & 0.01\\
&p4& 23 & 0 & 0& 0.9 & 0.01\\
&p5&23 &14.1 & 4.8 & 20.32 & 1.17 \\
    \bottomrule
 \multirow{5}{*}{Elevator} &p1&25 &0 &0& 0.71 &0.02\\ 
&p2&24 &0&0&0.73&0.04\\
&p3&25 &14&4.16&13.30&1.04\\
&p4&25 &0&0&0.70&0.03\\
&p5&24 &6.7&4.35&11.07 & 1.05\\
    \bottomrule
 \multirow{5}{*}{Logistics} &p1& 12 & 10.8 & 1.4& 8.7& 0.55\\ 
&p2&13& 0 & 0&0.78 & 0.03\\
&p3& 13 &0 & 0&0.78 & 0.03\\
&p4& 12& 9.8 & 2.2& 8.63 & 0.48 \\
&p5& 12 & 10.3 & 1.34 & 8.5 & 0.33\\
    \bottomrule
 \multirow{5}{*}{Rover} &p1&46 & 0 &0& 1.1 & 0.08\\ 
&p2&42 &0&0 &1.07 & 0.05\\
&p3&55 &0& 0&1.13 & 0.05\\
&p4&55 & 29.3 & 11.88 & 34.72 & 3.4\\
&p5& 69&0&0&4.74 & 0.07\\
    \bottomrule
  \end{tabular}
\caption{
Empirical evaluation of the proposed algorithm on a number of standard IPC domains)}
\label{tab1}
\end{table}

\section{Conclusion and Discussion}
\label{sec:disc}
In this paper, we present a reformulation of the value alignment problem, which explicitly accounts for an often overlooked aspect of the problem, namely the asymmetry between the human’s belief and the agent’s true model. Even in this setting, we see that value alignment, or more accurately the goal alignment problem remains a challenging one. We also see how we could leverage the human mental models to possibly generate better ways to query the human to find more information about their underlying objectives. Our initial empirical evaluation shows that even this approximate algorithm helps reduce the number of queries we would need to ask the human before the system can come up with a plan that is guaranteed to satisfy the true human goal. There are multiple ways this work could be extended. One possibility would be to extend the work to support more complex decision-making settings including decision-theoretic ones. Another one would be to look at the use of more realistic decision-making models for humans and also relax assumptions about access to the human mental model of the robot. While the value alignment problem is generally discussed in the context of AI safety, such misspecification and misalignment could affect every possible interaction between a human and AI agent. As such, we hope more researchers working in the area of human-AI interaction would try to account for such misalignment problems when designing their systems. 


\bibliographystyle{named}
\bibliography{ijcai23}

\end{document}